\documentclass{article}

\usepackage{etoolbox}
\newtoggle{withappendix}
\newtoggle{arxiv}
\toggletrue{withappendix}
\toggletrue{arxiv}

\iftoggle{arxiv}{
  \usepackage[numbers]{natbib}
  \setlength{\textwidth}{6.5in}
  \setlength{\textheight}{9in}
  \setlength{\oddsidemargin}{0in}
  \setlength{\evensidemargin}{0in}
  \setlength{\topmargin}{-0.5in}
  \newlength{\defbaselineskip}
  \setlength{\defbaselineskip}{\baselineskip}
  \setlength{\marginparwidth}{0.8in}
}{
  \PassOptionsToPackage{numbers}{natbib}
  \usepackage[final]{nips_2017}
}

\usepackage[utf8]{inputenc} 
\usepackage[T1]{fontenc}    
\usepackage{hyperref}       
\usepackage{url}            
\usepackage{booktabs}       
\usepackage{amsfonts}       
\usepackage{nicefrac}       
\usepackage{microtype}      
\usepackage{amsmath,amsthm,amssymb}

\usepackage{graphicx} 
\usepackage{caption}
\usepackage{subcaption}

\usepackage{color}

\usepackage{mathdots}

\usepackage{xspace}
\usepackage{tikz}
\usetikzlibrary{shapes,backgrounds}
\usetikzlibrary{arrows.meta}

\theoremstyle{definition}
\newtheorem{definition}{Definition}

\theoremstyle{plain}
\newtheorem{lemma}{Lemma}
\newtheorem{theorem}{Theorem}
\newtheorem{corollary}{Corollary}

\newcommand{\N}{\mathbb{N}}

\newcommand{\R}{\mathbb{R}}
\newcommand{\C}{\mathbb{C}}

\newcommand{\abs}[1]{\left| #1 \right| }
\newcommand{\E}[1]{\mathbf{E}\left[ #1 \right]}

\newcommand{\norm}[1]{\left\| #1 \right\|}
\DeclareMathOperator{\1}{\mathbf{1}}

\begin{document}

\title{Gaussian Quadrature for Kernel Features}

\iftoggle{arxiv}{
  \author{
    Tri Dao$^\dagger$,
    Christopher De Sa$^\ddagger$,
    Christopher R{\'e}$^\dagger$ \\
    $\dagger$ Stanford University,
    $\ddagger$ Cornell University \\
    \texttt{trid@stanford.edu},
    \texttt{cdesa@cs.cornell.edu},
    \texttt{chrismre@cs.stanford.edu}
  }
}
{
  \author{
    Tri Dao \\
    Department of Computer Science \\
    Stanford University \\
    Stanford, CA 94305 \\
    \texttt{trid@stanford.edu}
    \And
    Christopher De Sa \\
    Department of Computer Science \\
    Cornell University \\
    Ithaca, NY 14853 \\
    \texttt{cdesa@cs.cornell.edu}
    \And
    Christopher R{\'e} \\
    Department of Computer Science \\
    Stanford University \\
    Stanford, CA 94305 \\
    \texttt{chrismre@cs.stanford.edu}
  }
}

\maketitle

\begin{abstract}
  Kernel methods have recently attracted resurgent interest, showing performance
  competitive with deep neural networks in tasks such as speech recognition.
  The random Fourier features map is a technique commonly used to scale up
  kernel machines, but employing the randomized feature map means that
  $O(\epsilon^{-2})$ samples are required to achieve an approximation error of at most
  $\epsilon$.
  We investigate some alternative schemes for constructing feature maps that are
  deterministic, rather than random, by approximating the kernel in the
  frequency domain using Gaussian quadrature.
  We show that deterministic feature maps can be constructed, for any $\gamma > 0$,
  to achieve error $\epsilon$ with $O(e^{e^\gamma} + \epsilon^{-1/\gamma})$ samples as $\epsilon$ goes to 0.
  Our method works particularly well with sparse ANOVA kernels, which are
  inspired by the convolutional layer of CNNs.
  We validate our methods on datasets in different domains, such as MNIST and
  TIMIT, showing that deterministic features are faster to generate and achieve
  accuracy comparable to the state-of-the-art kernel methods based on random
  Fourier features.
\end{abstract}

\section{Introduction}

Kernel machines are frequently used to solve a wide variety of problems in
machine learning~\cite{scholkopf2002learning}.
They have gained resurgent interest and have recently been
shown~\cite{huang2014kernel, lu_how_2014, may2016compact, lu2016comparison,
  may2017kernel} to be competitive with deep neural networks in some tasks such
as speech recognition on large datasets.
A kernel machine is one that handles input $x_1, \dots, x_n$, represented as
vectors in $\R^d$, only in terms of some \emph{kernel function} $k: \R^d \times \R^d
\rightarrow \R$ of pairs of data points $k(x_i, x_j)$.
This representation is attractive for classification problems because one can
learn non-linear decision boundaries directly on the input without having to
extract features before training a linear classifier.

One well-known downside of kernel machines is the fact that they scale poorly to
large datasets.
Naive kernel methods, which operate on the \emph{Gram matrix} $G_{i,j} = k(x_i,
x_j)$ of the data, can take a very long time to run because the Gram matrix
itself requires $O(n^2)$ space and many operations on it (e.g., the singular
value decomposition) take up to $O(n^3)$ time.
\citet{rahimi2007random} proposed a solution to this problem: approximating the
kernel with an inner product in a higher-dimensional space.
Specifically, they suggest constructing a feature map $z: \R^d \rightarrow \R^D$ such that
$k(x, y) \approx \langle z(x), z(y) \rangle$.
This approximation enables kernel machines to use scalable linear methods for
solving classification problems and to avoid the pitfalls of naive kernel methods
by not materializing the Gram matrix.

In the case of shift-invariant kernels, one technique that was proposed for
constructing the function $z$ is \emph{random Fourier features}
\cite{rahimi2007random}.
This data-independent method approximates the Fourier transform
integral~\eqref{eq:fourier_integral} of the kernel by averaging Monte-Carlo
samples, which allows for arbitrarily-good estimates of the kernel function $k$.
\citet{rahimi2007random} proved that if the feature map has dimension $D =
\tilde \Omega\left( \frac{d}{\epsilon^2} \right)$ then, with constant probability, the
approximation $\langle z(x), z(y) \rangle$ is uniformly $\epsilon$-close to the true kernel on a
bounded set.
While the random Fourier features method has proven to be effective in solving
practical problems, it comes with some caveats.
Most importantly, the accuracy guarantees are only probabilistic and there is no
way to easily compute, for a particular random sample, whether the desired
accuracy is achieved.

Our aim is to understand to what extent randomness is necessary to approximate a
kernel.
We thus propose a fundamentally different scheme for constructing the feature
map $z$.
While still approximating the kernel's Fourier transform
integral~\eqref{eq:fourier_integral} with a discrete sum, we select the sample
points and weights \emph{deterministically}.
This gets around the issue of probabilistic-only guarantees by removing the
randomness from the algorithm.
For small dimension, deterministic maps yield significantly lower error.
As the dimension increases, some random sampling may become necessary, and our
theoretical insights provide a new approach to sampling.
Moreover, for a particular class of kernels called sparse ANOVA kernels (also
known as convolutional kernels as they are similar to the convolutional layer in
CNNs) which have shown state-of-the-art performance in speech
recognition~\cite{may2017kernel}, deterministic maps require fewer samples than
random Fourier features, both in terms of the desired error and the kernel size.
We make the following contributions:
\begin{itemize}
  \setlength{\itemsep}{0pt}
  \item In Section~\ref{sec:kernels_and_quadrature}, we describe how to
  deterministically construct a feature map $z$ for the class of subgaussian
  kernels (which can approximate any kernel well) that has exponentially small
  (in $D$) approximation error.
  \item In Section~\ref{sec:sparse_anova_kernels}, for sparse ANOVA kernels,
  we show that our method produces good estimates using only $O(d)$ samples,
  whereas random Fourier features requires $O(d^3)$ samples.
  \item In Section~\ref{sec:experiments}, we validate our results
  experimentally.
  We demonstrate that, for real classification problems on MNIST and TIMIT
  datasets, our method combined with random sampling yields up to 3 times lower
  kernel approximation error.
  With sparse ANOVA kernels, our method slightly improves classification
  accuracy compared to the state-of-the-art kernel methods based on random
  Fourier features (which are already shown to match the performance of deep
  neural networks), all while speeding up the feature generation process.
\end{itemize}

\section{Related Work}
\label{sec:related_work}

Much work has been done on extracting features for kernel methods.
The random Fourier features method has been analyzed in the context of several
learning algorithms, and its generalization error has been characterized and
compared to that of other kernel-based algorithms~\cite{rahimi2009weighted}.
It has also been compared to the Nystr\"{o}m method~\cite{yang2012nystrom},
which is data-dependent and thus can sometimes outperform random Fourier
features.
Other recent work has analyzed the generalization performance of the random
Fourier features algorithm~\cite{lin2014sample}, and improved the bounds on its
maximum error~\cite{sriperumbudur2015rates,sutherland2015error}.

While we focus here on deterministic approximations to the Fourier transform
integral and compare them to Monte Carlo estimates, these are not the only two
methods available to us.
A possible middle-ground method is \emph{quasi-Monte Carlo} estimation, in which
low-discrepancy sequences, rather than the fully-random samples of Monte Carlo
estimation, are used to approximate the integral.
This approach was analyzed in \citet{yang2014quasi} and shown to achieves an
asymptotic error of $\epsilon = O\left( D^{-1} \left( \log(D) \right)^d \right)$.
While this is asymptotically better than the random Fourier features method, the
complexity of the quasi-Monte Carlo method coupled with its larger constant
factors prevents it from being strictly better than its predecessor.
Our method still requires asymptotically fewer samples as $\epsilon$ goes to 0.

Our deterministic approach here takes advantage of a long line of work on
numerical quadrature for estimating integrals.
\citet{bach2015equivalence} analyzed in detail the connection between quadrature
and random feature expansions, thus deriving bounds for the number of samples
required to achieve a given average approximation error (though they did not
present complexity results regarding maximum error nor suggested new feature
maps).
This connection allows us to leverage longstanding deterministic numerical
integration methods such as Gaussian
quadrature~\cite{gauss1815methodus,trefethen2008gauss} and sparse
grids~\cite{bungartz2004sparse}.

Unlike many other kernels used in machine learning, such as the Gaussian kernel,
the sparse ANOVA kernel allows us to encode prior information about the
relationships among the input variables into the kernel itself.
Sparse ANOVA kernels have been shown~\cite{stitson1999support} to work
well for many classification tasks, especially in structural modeling problems
that benefit from both the good generalization of a kernel machine and the
representational advantage of a sparse model~\cite{gunn2002structural}.

\section{Kernels and Quadrature}
\label{sec:kernels_and_quadrature}

We start with a brief overview of kernels.
A kernel function $k \colon \mathbb{R}^d \times \mathbb{R}^d \to \mathbb{R}$ encodes
the \emph{similarity} between pairs of examples.
In this paper, we focus on shift invariant kernels (those which satisfy $k(x, y)
= k(x - y)$, where we overload the definition of $k$ to also refer to a function
$k: \R^d \rightarrow \R$) that are positive definite and properly scaled.
A kernel is positive definite if its Gram matrix is always positive definite for
all non-trivial inputs, and it is properly-scaled if $k(x, x) = 1$ for all $x$.
In this setting, our results make use of a theorem~\cite{rudin1990fourier} that
also provides the ``key insight'' behind the random Fourier features method.

\begin{theorem}[Bochner's theorem]
  \label{thmBochner}
  A continuous shift-invariant properly-scaled kernel $k: \R^d \times \R^d \rightarrow \R$ is
  positive definite if and only if $k$ is the Fourier transform of a proper
  probability distribution.
\end{theorem}

We can then write $k$ in terms of its Fourier transform $\Lambda$ (which is a proper
probability distribution):
\begin{equation}
  k(x - y) = \int_{\R^d} \Lambda(\omega) \exp(j \omega^\top (x - y)) \, d \omega.
  \label{eq:fourier_integral}
\end{equation}
For $\omega$ distributed according to $\Lambda$, this is equivalent to writing
\begin{equation*}
  k(x - y) = \E{\exp(j \omega^\top (x - y))} = \E{ \langle \exp(j \omega^\top x), \exp(j \omega^\top y) \rangle },
\end{equation*}
where we use the usual Hermitian inner product $\langle x, y \rangle = \sum_{i} x_i
\overline{y_i}$.
The random Fourier features method proceeds by estimating this expected value
using Monte Carlo sampling averaged across $D$ random selections of $\omega$.
Equivalently, we can think of this as approximating~\eqref{eq:fourier_integral}
with a discrete sum at randomly selected sample points.

Our objective is to choose some points $\omega_i$ and weights $a_i$ to uniformly
approximate the integral~\eqref{eq:fourier_integral} with $\tilde{k}(x - y) =
\sum_{i=1}^D a_i \exp(j \omega_j^\top(x - y))$.
To obtain a \emph{feature map} $z: \R^d \rightarrow \C^D$ where $\tilde{k} (x - y) =
\sum_{i=1}^D a_i z_i(x) \overline{z_i(y)}$, we can define
\begin{equation*}
  z(x) = \begin{bmatrix} \sqrt{a_1} \exp(j \omega_1^\top x) & \dots & \sqrt{a_D} \exp(j
    \omega_D^\top x) \end{bmatrix}^\top.
\end{equation*}
We aim to bound the maximum error for $x, y$ in a region $\mathcal{M}$ with
diameter $M = \sup_{x, y \in \mathcal{M}} \norm{x - y}$:
\begin{equation}
  \epsilon = \sup_{(x, y) \in \mathcal{M}} \abs{k(x - y) - \tilde{k}(x - y)}
   = \sup_{\norm{u} \le M} \abs{\int_{\R^d} \Lambda(\omega) e^{j \omega^\top u} \, d \omega - \sum_{i=1}^D a_i
    e^{j \omega_i^\top u}}.
  \label{eqnMaximumError}
\end{equation}

A \emph{quadrature rule} is a choice of $\omega_i$ and $a_i$ to minimize this maximum
error.
To evaluate a quadrature rule, we are concerned with the \emph{sample
  complexity} (for a fixed diameter $M$).
\begin{definition}
  \label{defnSampleComplexity}
  For any $\epsilon > 0$, a quadrature rule has sample complexity $D_{\mathrm{SC}}(\epsilon) =
  D$, where $D$ is the smallest value such that the rule, when instantiated with
  $D$ samples, has maximum error at most $\epsilon$.
\end{definition}

We will now examine ways to construct deterministic quadrature rules and their
sample complexities.

\subsection{Gaussian Quadrature}

Gaussian quadrature is one of the most popular techniques in one-dimensional
numerical integration.
The main idea is to approximate integrals of the form $\int \Lambda(\omega) f(\omega) \, d \omega \approx
\sum_{i=1}^D a_i f(\omega_i)$ such that the approximation is exact for all polynomials
below a certain degree; $D$ points are sufficient for polynomials of degree up
to $2D - 1$.
While the points and weights used by Gaussian quadrature depend both on the
distribution $\Lambda$ and the parameter $D$, they can be computed efficiently using
orthogonal polynomials~\cite{hale2013fast,townsend2015fast}.
Gaussian quadrature produces accurate results when integrating functions that
are well-approximated by polynomials, which include all subgaussian densities.
\begin{definition}[Subgaussian Distribution]
  We say that a distribution $\Lambda: \R^d \rightarrow \R$ is \emph{subgaussian} with parameter
  $b$ if for $X \sim \Lambda$ and for all $t \in \R^d$, $\E{\exp(\langle t, X \rangle)} \le \exp\left(
    \frac{1}{2} b^2 \norm{t}^2 \right)$.
\end{definition}
We subsequently assume that the distribution $\Lambda$ is subgaussian, which is a
technical restriction compared to random Fourier features.
Many of the kernels encountered in practice have subgaussian spectra, including
the ubiquitous Gaussian kernel.
More importantly, we can approximate any kernel by convolving it with the
Gaussian kernel, resulting in a subgaussian kernel.
The approximation error can be made much smaller than the inherent noise in the
data generation process.

\subsection{Polynomially-Exact Rules}
\label{ssPolyExactRules}

Since Gaussian quadrature is so successful in one dimension, as commonly done in
the numerical analysis literature \cite{isaacson1994analysis}, we might consider
using quadrature rules that are multidimensional analogues of Gaussian
quadrature --- rules that are accurate for all polynomials up to a certain degree
$R$.
In higher dimensions, this is equivalent to saying that our quadrature rule
satisfies
\begin{equation}
  \int_{\R^d} \Lambda(\omega) \prod_{l=1}^d (e_l^\top \omega)^{r_l} \, d \omega
  = \sum_{i=1}^D a_i \prod_{l=1}^d (e_l^\top \omega_i)^{r_l}
  \quad \text{for all } r \in \N^d \text{ such that } \sum_l r_l \leq R,
  \label{eqnPolyExactConstraints}
\end{equation}
where $e_l$ are the standard basis vectors.

To test the accuracy of polynomially-exact quadrature, we constructed a feature
map for a Gaussian kernel, $\Lambda(\omega) = (2 \pi)^{-\frac{d}{2}} \exp\left( - \frac{1}{2}
  \norm{\omega}^2 \right)$, in $d = 25$ dimensions with $D = 1000$ and accurate for
all polynomials up to degree $R = 2$.
In Figure~\ref{fig:polynomially_exact}, we compared this to a random Fourier
features rule with the same number of samples, over a range of region diameters
$M$ that captures most of the data points in practice (as the kernel is properly
scaled).
For small regions in particular, a polynomially-exact scheme can have a
significantly lower error than a random Fourier feature map.

\begin{figure}[t]
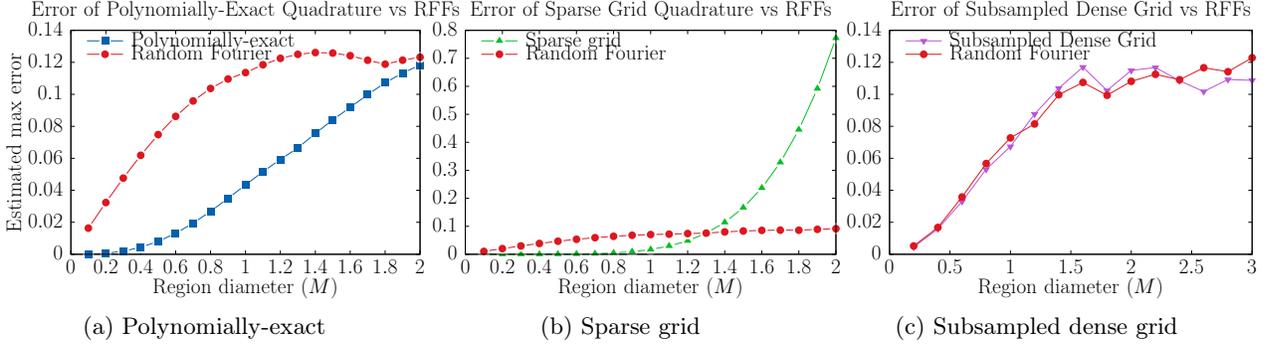

\centering
  \begin{subfigure}{0.33\linewidth}
    \centering
    \resizebox{!}{0.75\textwidth}{
    \LARGE \input{plotmaxerr.tex}
    }
    \caption{Polynomially-exact}
    \label{fig:polynomially_exact}
  \end{subfigure}\hfill%
  \begin{subfigure}{0.33\linewidth}
    \centering
    \resizebox{!}{0.75\textwidth}{
    \LARGE \input{plotmaxerrsg.tex}
    }
    \caption{Sparse grid}
    \label{fig:sparse_grid}
  \end{subfigure}\hfill%
  \begin{subfigure}{0.33\linewidth}
    \centering
    \resizebox{!}{0.75\textwidth}{
    \LARGE \input{plotmaxerrsubsamp.tex}
    }
    \caption{Subsampled dense grid}
    \label{fig:subsampled_dense_grid}
  \end{subfigure}\hfill%
  \caption{Error comparison (empirical maximum over $10^6$ uniformly-distributed
    samples) of different quadrature schemes and the random Fourier features
    method.}
\label{fig:error_comparison}
\end{figure}

This experiment motivates us to investigate theoretical bounds on the behavior
of this method.
For subgaussian kernels, it is straightforward to bound the maximum error of a
polynomially-exact feature map using the Taylor series approximation of the
exponential function in~\eqref{eqnMaximumError}.
\begin{theorem}
\label{thmPolynomialQuadrature}
Let $k$ be a kernel with $b$-subgaussian spectrum, and let $\tilde{k}$ be its
estimation under some quadrature rule with non-negative weights that is exact up
to some even degree $R$.
Let $\mathcal{M} \subset \R^d$ be some region of diameter $M$.
Then, for all $x, y \in \mathcal{M}$, the error of the quadrature features
approximation is bounded by
\begin{equation*}
  \abs{k(x - y) - \tilde{k}(x - y)} \le 3 \left(\frac{e b^2 M^2}{R}
  \right)^{\frac{R}{2}}.
\end{equation*}
\end{theorem}
All the proofs are found in the Appendix.

To bound the sample complexity of polynomially-exact quadrature, we need to
determine how many quadrature samples we will need to satisfy the conditions of
Theorem~\ref{thmPolynomialQuadrature}.
There are $\binom{d + R}{d}$ constraints in~\eqref{eqnPolyExactConstraints}, so
a series of polynomially-exact quadrature rules that use only about this many
sample points can yield a bound on the sample complexity of this quadrature
rule.

\begin{corollary}
\label{corPolynomialQuadrature}
Assume that we are given a class of feature maps that satisfy the conditions of
Theorem~\ref{thmPolynomialQuadrature}, and that all have a number of samples $D
\le \beta {d + R \choose d}$ for some fixed constant $\beta$.
Then, for any $\gamma > 0$, the sample complexity of features maps in this class can
be bounded by
\begin{equation*}
  D(\epsilon)
  \le \beta 2^d \max\left(\exp\left(e^{2 \gamma + 1} b^2 M^2 \right), \left( \frac{3}{\epsilon}
    \right)^{\frac{1}{\gamma}} \right).
\end{equation*}
In particular, for a fixed dimension $d$, this means that for any $\gamma$, $D(\epsilon) =
O\left(\epsilon^{-\frac{1}{\gamma}} \right)$.
\end{corollary}

The result of this corollary implies that, in terms of the desired error $\epsilon$,
the sample complexity increases asymptotically slower than any negative power of
$\epsilon$.
Compared to the result for random Fourier features which had $D(\epsilon) = O(\epsilon^{-2})$,
this has a much weaker dependence on $\epsilon$.
While this weaker dependence does come at the cost of an additional factor of
$2^d$, it is a constant cost of operating in dimension $d$, and is not dependent
on the error $\epsilon$.

The more pressing issue, when comparing polynomially-exact features to random
Fourier features, is the fact that we have no way of efficiently constructing
quadrature rules that satisfy the conditions of
Theorem~\ref{thmPolynomialQuadrature}.
One possible construction involves selecting random sample points $\omega_i$, and
then solving~\eqref{eqnPolyExactConstraints} for the values of $a_i$ using a
non-negative least squares (NNLS) algorithm.
While this construction works in low dimensions --- it is the method we used for
the experiment in Figure~\ref{fig:polynomially_exact} --- it rapidly becomes
infeasible to solve for higher values of $d$ and $R$.

We will now show how to overcome this issue by introducing quadrature rules that
can be rapidly constructed using grid-based quadrature rules.
These rules are constructed directly from products of a one-dimensional
quadrature rule, such as Gaussian quadrature, and so avoid the
construction-difficulty problems encountered in this section.
Although grid-based quadrature rules can be constructed for any kernel
function~\cite{bungartz2004sparse}, they are easier to conceptualize when the
kernel $k$ factors along the dimensions, as $k(u) = \prod_{i=1}^d k_i(u_i)$.
For simplicity we will focus on this factorizable case.

\subsection{Dense Grid Quadrature}

The simplest way to do this is with a \emph{dense grid} (also known as tensor
product) construction.
A dense grid construction starts by factoring the integral
(\ref{eq:fourier_integral}) into $k(u) = \prod_{i=1}^d \left(\int_{-\infty}^{\infty} \Lambda_i(\omega)
  \exp(j \omega e_i^\top u) \, d \omega \right)$, where $e_i$ are the standard basis vectors.
Since each of the factors is an integral over a single dimension, we can
approximate them all with a one-dimensional quadrature rule.
In this paper, we focus on Gaussian quadrature, although we could also use other
methods such as Clenshaw-Curtis~\cite{clenshaw1960method}.
Taking tensor products of the points and weights results in the dense grid
quadrature.
The detailed construction is given in
Appendix~\ref{sec:dense_grid_construction}.

The individual Gaussian quadrature rules are exact for all polynomials up to
degree $2L - 1$, so the dense grid is also accurate for all such polynomials.
Theorem~\ref{thmPolynomialQuadrature} then yields a bound on its sample
complexity.

\begin{corollary}
\label{corDenseGridQuadrature}
Let $k$ be a kernel with a spectrum that is subgaussian with parameter $b$.
Then, for any $\gamma > 0$, the sample complexity of dense grid features can be
bounded by
\begin{equation*}
  D(\epsilon) \le \max\left(\exp\left( d e^{\gamma d} \frac{e b^2 M^2}{2} \right), \left(
      \frac{3}{\epsilon} \right)^{\frac{1}{\gamma}} \right).
\end{equation*}
In particular, as was the case with polynomially-exact features, for a fixed
$d$, $D(\epsilon) = O\left(\epsilon^{-\frac{1}{\gamma}} \right)$.
\end{corollary}

Unfortunately, this scheme suffers heavily from the curse of dimensionality,
since the sample complexity is doubly-exponential in $d$.
This means that, even though they are easy to compute, the dense grid method
does not represent a useful solution to the issue posed in Section
\ref{ssPolyExactRules}.

\subsection{Sparse Grid Quadrature}

The curse of dimensionality for quadrature in high dimensions has been studied
in the numerical integration setting for decades.
One of the more popular existing techniques for getting around the curse is
called \emph{sparse grid} or Smolyak quadrature~\cite{smolyak1963}, originally
developed to solve partial differential equations.
Instead of taking the tensor product of the one-dimensional quadrature rule, we
only include points up to some fixed total level $A$, thus constructing a linear
combination of dense grid quadrature rules that achieves a similar error with
exponentially fewer points than a single larger quadrature rule.
The detailed construction is given in
Appendix~\ref{sec:sparse_grid_construction}.
Compared to polynomially-exact rules, sparse grid quadrature can be computed
quickly and easily (see Algorithm 4.1 from \citet{holtz2010sparse}).

To measure the performance of sparse grid quadrature, we constructed a
feature map for the same Gaussian kernel analyzed in the previous section,
with $d = 25$ dimensions and up to level $A = 2$.
We compared this to a random Fourier features rule with the
same number of samples, $D = 1351$, and plot the results in Figure
\ref{fig:sparse_grid}.  As was the case with polynomially-exact quadrature,
this sparse grid scheme has tiny error for small-diameter regions, but this
error unfortunately increases to be even larger than that of random Fourier
features as the region diameter increases.

The sparse grid construction yields a bound on the sample count: $D \leq 3^A
\binom{d + A}{A}$, where $A$ is the bound on the total level.
By extending known bounds on the error of Gaussian quadrature, we can similarly
bound the error of the sparse grid feature method.
\begin{theorem}
\label{thmSparseGridQuadrature}
Let $k$ be a kernel with a spectrum that is subgaussian with parameter $b$, and
let $\tilde{k}$ be its estimation under the sparse grid quadrature rule up to level
$A$.
Let $\mathcal{M} \subset \R^d$ be some region of diameter $M$, and assume that $A \ge 24e
b^2 M^2$.
Then, for all $x, y \in \mathcal{M}$, the error of the quadrature features
approximation is bounded by
\begin{equation*}
  \abs{k(x - y) - \tilde{k}(x - y)} \le 2^d \left(\frac{12e b^2 M^2}{A} \right)^{A}.
\end{equation*}
\end{theorem}

This, along with our above upper bound on the sample count, yields a bound on
the sample complexity.
\begin{corollary}
\label{corSparseGridQuadrature}
Let $k$ be a kernel with a spectrum that is subgaussian with parameter $b$.
Then, for any $\gamma > 0$, the sample complexity of sparse grid features can be
bounded by
\[
  D(\epsilon) \le 2^{d} \max\left(\exp\left(24 e^{2 \gamma+1} b^2 M^2 \right), 2^{\frac{d}{\gamma}}
    \epsilon^{-\frac{1}{\gamma}} \right).
\]
As was the case with all our previous deterministic features maps, for a fixed
$d$, $D(\epsilon) = O\left(\epsilon^{-\frac{1}{\gamma}} \right)$.
\end{corollary}

\paragraph{Subsampled grids}
One of the downsides of the dense/sparse grids analyzed above is the difficulty
of tuning the number of samples extracted in the feature map.
As the only parameter we can typically set is the degree of polynomial
exactness, even a small change in this (e.g., from 2 to 4) can produce a
significant increase in the number of features.
However, we can always subsample the grid points according to the distribution
determined by their weights to both tame the curse of dimensionality and to have
fine-grained control over the number of samples.
For simplicity, we focus on subsampling the dense grid.
In Figure~\ref{fig:subsampled_dense_grid}, we compare the empirical errors of
subsampled dense grid and random Fourier features, noting that they are
essentially the same across all diameters.

\subsection{Reweighted Grid Quadrature}
\label{sub:reweighted_quadrature}

Both random Fourier features and dense/sparse grid quadratures are
data-independent.
We now describe a data-adaptive method to choose a quadrature for a
pre-specified number of samples: reweighting the grid points to minimize the
difference between the approximate and the exact kernel on a small subset of
data.
Adjusting the grid to the data distribution yields better kernel approximation.

We approximate the kernel $k(x - y)$ with
\begin{equation*}
  \tilde{k}(x - y) = \sum_{i=1}^D a_i \exp(j \omega_i^\top (x - y)) = \sum_{i=1}^{D} a_i \cos(\omega_i^\top(x - y)),
\end{equation*}
where $a_i \geq 0$, as $k$ is real-valued.
We first choose the set of potential grid points $\omega_1, \dots, \omega_D$ by sampling
from a dense grid of Gaussian quadrature points.
To solve for the weights $a_1, \dots, a_D$, we independently sample $n$ pairs
$(x_1, y_1), \dots, (x_n, y_n)$ from the dataset, then minimize the empirical
mean squared error (with variable $a_1, \dots, a_D$):
\begin{equation*}
  \begin{array}{ll}
    \mbox{minimize} & \displaystyle \frac{1}{n} \sum_{l=1}^{n} \left( k(x_l - y_l) -
                        \tilde{k}(x_l - y_l) \right)^2 \\
    \mbox{subject to} & a_i \geq 0, \text{ for } i = 1, \dots, D.
  \end{array}
\end{equation*}

For appropriately defined matrix $M$ and vector $b$, this is an NNLS problem of
minimizing $\frac{1}{n} \norm{Ma - b}^2$ subject to $a \geq 0$, with variable $a \in
\mathbb{R}^D$.
The solution is often sparse, due to the active elementwise constraints $a \geq 0$.
Hence we can pick a larger set of potential grid points $\omega_1, \dots, \omega_{D'}$
(with $D' > D$) and solve the above problem to obtain a smaller set of grid
points (those with $a_j > 0$).
To get even sparser solution, we add an $\ell_1$-penalty term with parameter $\lambda \geq
0$:
\begin{equation*}
  \begin{array}{ll}
    \mbox{minimize} & \frac{1}{n} \norm{Ma - b}^2 + \lambda \1^\top a \\
    \mbox{subject to} & a_i \geq 0, \text{ for } i = 1, \dots, D'.
  \end{array}
\end{equation*}
Bisecting on $\lambda$ yields the desired number of grid points.

As this is a data-dependent quadrature, we empirically evaluate its performance
on the TIMIT dataset, which we will describe in more details in
Section~\ref{sec:experiments}.
In Figure~\ref{fig:rms_timit}, we compare the estimated root mean squared error
on the dev set of different feature generation schemes against the number of
features $D$ (mean and standard deviation over 10 runs).
Random Fourier features, Quasi-Monte Carlo (QMC) with Halton sequence, and
subsampled dense grid have very similar approximation error, while reweighted
quadrature has much lower approximation error.
Reweighted quadrature achieves 2--3 times lower error for the same number of
features and requiring 3--5 times fewer features for a fixed threshold of
approximation error compared to random Fourier features.
Moreover, reweighted features have extremely low variance, even though the
weights are adjusted based only on a very small fraction of the dataset (500
samples out of 1 million data points).

\paragraph{Faster feature generation}
Not only does grid-based quadrature yield better statistical performance to
random Fourier features, it also has some notable systems benefits.
Generating quadrature features requires a much smaller number of multiplies, as
the grid points only take on a finite set of values for all dimensions (assuming
an isotropic kernel).
For example, a Gaussian quadrature that is exact up to polynomials of degree 21
only requires 11 grid points for each dimension.
To generate the features, we multiply the input with these 11 numbers before
adding the results to form the deterministic features.
The save in multiples may be particularly significant in architectures such as
application-specific integrated circuits (ASICs).
In our experiment on the TIMIT dataset in Section~\ref{sec:experiments}, this
specialized matrix multiplication procedure (on CPU) reduces the feature
generation time in half.

\section{Sparse ANOVA Kernels}
\label{sec:sparse_anova_kernels}

One type of kernel that is commonly used in machine learning, for example in
structural modeling, is the \emph{sparse ANOVA kernels}~\cite{hofmann2008kernel,
  gunn_structural_2002}.
They are also called \emph{convolutional kernels}, as they operate similarly to
the convolutional layer in CNNs.
These kernels have achieved state-of-the-art performance on large real-world
datasets \cite{lu_how_2014, may2017kernel}, as we will see in
Section~\ref{sec:experiments}.
A kernel of this type can be written as
\[
  k(x, y) = \sum_{S \in \mathcal{S}} \prod_{i \in S} k_1(x_i - y_i),
\]
where $\mathcal{S}$ is a set of subsets of the variables in $\{1, \ldots, d\}$, and
$k_1$ is a one-dimensional kernel.
(Straightforward extensions, which we will not discuss here, include using
different one-dimensional kernels for each element of the products, and
weighting the sum.)
Sparse ANOVA kernels are used to encode sparse dependencies among the variables:
two variables are related if they appear together in some $S \in \mathcal{S}$.
These sparse dependencies are typically problem-specific: each $S$ could
correspond to a factor in the graph if we are analyzing a distribution modeled
with a factor graph.
Equivalently, we can think of the set $\mathcal{S}$ as a hypergraph, where each
$S \in \mathcal{S}$ corresponds to a hyperedge.
Using this notion, we define the \emph{rank} of an ANOVA kernel to be $r =
\max_{S \in \mathcal{S}} \abs{S}$, the \emph{degree} as $\Delta = \max_{i \in \{1, \ldots,
  d\}} \abs{\left\{ S \in \mathcal{S} \middle| i \in S \right\}}$, and the
\emph{size} of the kernel to be the number of hyperedges $m =
\abs{\mathcal{S}}$.
For sparse models, it is common for both the rank and the degree to be small,
even as the number of dimensions $d$ becomes large, so $m = O(d)$.
This is the case we focus on in this section.

It is straightforward to apply the random Fourier features method to construct
feature maps for ANOVA kernels: construct feature maps for each of the (at most
$r$-dimensional) sub-kernels $k_{S}(x - y) = \prod_{i \in S} k_1(x_i - y_i)$
individually, and then combine the results.
To achieve overall error $\epsilon$, it suffices for each of the sub-kernel feature
maps to have error $\epsilon / m$; this can be achieved by random Fourier features
using $D_S = \tilde \Omega\left(r (\epsilon m^{-1})^{-2} \right) = \tilde \Omega\left(r m^2
  \epsilon^{-2} \right)$ samples each, where the notation $\tilde \Omega$ hides the $\log
1/\epsilon$ factor.
Summed across all the $m$ sub-kernels, this means that the random Fourier
features map can achieve error $\epsilon$ with constant probability with a sample
complexity of $D(\epsilon) = \tilde \Omega\left( r m^3 \epsilon^{-2} \right)$ samples.
While it is nice to be able to tackle this problem using random features, the
cubic dependence on $m$ in this expression is undesirable: it is significantly
larger than the $D = \tilde \Omega( d \epsilon^{-2} )$ we get in the non-ANOVA case.

Can we construct a deterministic feature map that has a better error bound?
It turns out that we can.

\begin{theorem}
  \label{thmSparseANOVAKernel}
  Assume that we use polynomially-exact quadrature to construct features for
  each of the sub-kernels $k_{S}$, under the conditions of
  Theorem~\ref{thmPolynomialQuadrature}, and then combine the resulting feature
  maps to produce a feature map for the full ANOVA kernel.
  For any $\gamma > 0$, the sample complexity of this method is
  \begin{equation*}
    D(\epsilon) \le \beta m 2^r \max\left(\exp\left( e^{2\gamma+1} b^2 M^2 \right),
      (3\Delta)^{\frac{1}{\gamma}} \epsilon^{-\frac{1}{\gamma}} \right).
  \end{equation*}
\end{theorem}
Compared to the random Fourier features, this rate depends only linearly on $m$.
For fixed parameters $\beta$, $b$, $M$, $\Delta$, $r$, and for any $\gamma > 0$, we can bound
the sample complexity $D(\epsilon) = O(m \epsilon^{-\frac{1}{\gamma}})$, which is better than
random Fourier features \emph{both} in terms of the kernel size $m$ and the
desired error $\epsilon$.

\section{Experiments}
\label{sec:experiments}

To evaluate the performance of deterministic feature maps, we analyzed the
accuracy of a sparse ANOVA kernel on the MNIST digit classification
task~\cite{lecun1998gradient} and the TIMIT speech recognition
task~\cite{timit}.

\paragraph{Digit classification on MNIST}

This task consists of $70,000$ examples ($60,000$ in the training dataset and
$10,000$ in the test dataset) of hand-written digits which need to be
classified.
Each example is a $28 \times 28$ gray-scale image.
Clever kernel-based SVM techniques are known to achieve very low error rates
(e.g., $0.79\%$) on this problem~\cite{maji2009fast}.
We do not attempt to compare ourselves with these rates; rather, we compare
random Fourier features and subsampled dense grid features that both approximate
the same ANOVA kernel.
The ANOVA kernel we construct is designed to have a similar structure to the
first layer of a convolutional neural network~\cite{simard2003best}.
Just as a filter is run on each $5 \times 5$ square of the image, for our ANOVA
kernel, each of the sub-kernels is chosen to run on a $5 \times 5$ square of the
original image (note that there are many, $(28 - 5 + 1)^2 = 576$, such squares).
We choose the simple Gaussian kernel as our one-dimensional kernel.

Figure~\ref{fig:accuracy_mnist} compares the dense grid subsampling method to
random Fourier features across a range of feature counts.
The deterministic feature map with subsampling performs better than the random
Fourier feature map across most large feature counts, although its performance
degrades for very small feature counts.
The deterministic feature map is also somewhat faster to compute, taking---for the
$28800$-features---320 seconds vs.\ 384 seconds for the random Fourier
features, a savings of $17\%$.

\paragraph{Speech recognition on TIMIT}

This task requires producing accurate transcripts from raw audio recordings of
conversations in English, involving 630 speakers, for a total of 5.4 hours of
speech.
We use the kernel features in the acoustic modeling step of speech recognition.
Each data point corresponds to a frame (10ms) of audio data, preprocessed using
the standard feature space Maximum Likelihood Linear Regression (fMMLR)
\cite{gales1998maximum}.
The input $x$ has dimension 40.
After generating kernel features $z(x)$ from this input, we model the
corresponding phonemes $y$ by a multinomial logistic regression model.
Again, we use a sparse ANOVA kernel, which is a sum of 50 sub-kernels of the
form $\exp(-\gamma \norm{x_S - y_S}^2)$, each acting on a subset $S$ of 5 indices.
These subsets are randomly chosen a priori.
To reweight the quadrature features, we sample 500 data points out of 1 million.

We plot the phone error rates (PER) of a speech recognizer trained based on
different feature generation schemes against the number of features $D$ in
Figure~\ref{fig:per_timit} (mean and standard deviation over 10 runs).
Again, subsampled dense grid performs similarly to random Fourier features, QMC
yields slightly higher error, while reweighted features achieve slightly lower
phone error rates.
All four methods have relatively high variability in their phone error rates due
to the stochastic nature of the training and decoding steps in the speech
recognition pipeline.
The quadrature-based features (subsampled dense grids and reweighted quadrature)
are about twice as fast to generate, compared to random Fourier features, due to
the small number of multiplies required.
We use the same setup as \citet{may2017kernel}, and the performance here matches
both that of random Fourier features and deep neural networks
in~\citet{may2017kernel}.

\begin{figure}[t]
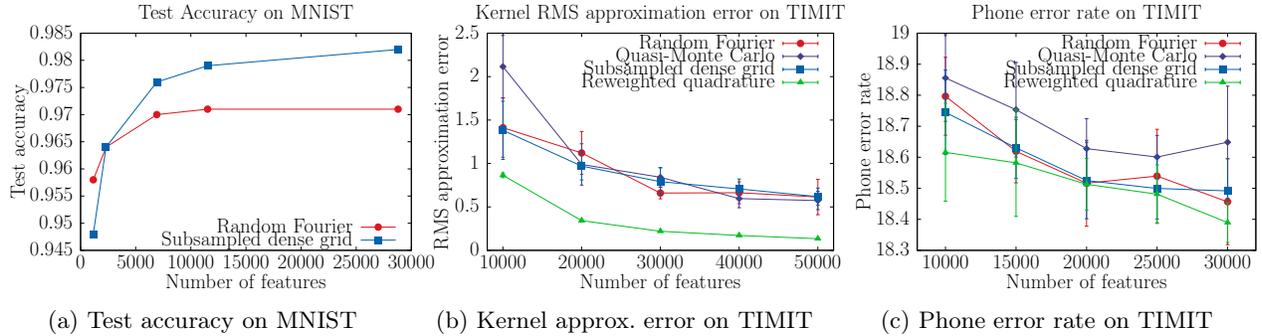

  \centering
  \begin{subfigure}{0.32\textwidth}
    \centering
    \resizebox{!}{0.75\columnwidth}{
    \LARGE \input{plotmnist.tex}
    }
    \caption{Test accuracy on MNIST}
    \label{fig:accuracy_mnist}
  \end{subfigure}\hfill%
  \begin{subfigure}{0.32\textwidth}
    \centering
    \resizebox{!}{0.75\columnwidth}{
    \LARGE \input{plotrmserror.tex}
    }
    \caption{Kernel approx.\ error on TIMIT}
    \label{fig:rms_timit}
  \end{subfigure}\hfill%
  \begin{subfigure}{0.32\textwidth}
    \centering
    \resizebox{!}{0.75\columnwidth}{
    \LARGE \input{plotwertimit.tex}
    }
    \caption{Phone error rate on TIMIT}
    \label{fig:per_timit}
  \end{subfigure}\hfill%
  \caption{Performance of different feature generation schemes on MNIST and TIMIT.}
  \label{fig:performance_timit}
\end{figure}

\section{Conclusion}

We presented deterministic feature maps for kernel machines.
We showed that we can achieve better scaling in the desired accuracy $\epsilon$
compared to the state-of-the-art method, random Fourier features.
We described several ways to construct these feature maps, including
polynomially-exact quadrature, dense grid construction, sparse grid
construction, and reweighted grid construction.
Our results apply well to the case of sparse ANOVA kernels, achieving
significant improvements (in the dependency on the dimension $d$) over random
Fourier features.
Finally, we evaluated our results experimentally, and showed that ANOVA kernels
with deterministic feature maps can produce comparable accuracy to the
state-of-the-art methods based on random Fourier features on real datasets.

ANOVA kernels are an example of how structure can be used to define better
kernels.
Resembling the convolutional layers of convolutional neural networks, they
induce the necessary inductive bias in the learning process.
Given CNNs' recent success in other domains beside images, such as sentence
classification~\cite{kim2014convolutional} and machine
translation~\cite{gehring2017convolutional}, we hope that our work on
deterministic feature maps will enable kernel methods such as ANOVA kernels to
find new areas of application.

\subsubsection*{Acknowledgments}

This material is based on research sponsored by Defense Advanced Research
Projects Agency (DARPA) under agreement number FA8750-17-2-0095.
We gratefully acknowledge the support of the DARPA SIMPLEX program under No.\ N66001-15-C-4043,
DARPA FA8750-12-2-0335 and FA8750-13-2-0039,
DOE 108845,
National Institute of Health (NIH) U54EB020405,
the National Science Foundation (NSF) under award No.\ CCF-1563078,
the Office of Naval Research (ONR) under awards No.\ N000141210041 and No.\ N000141310129,
the Moore Foundation,
the Okawa Research Grant,
American Family Insurance,
Accenture,
Toshiba, and Intel.
This research was supported in part by affiliate members and other supporters of
the Stanford DAWN project: Intel, Microsoft, Teradata, and VMware.
The U.S.\ Government is authorized to reproduce and distribute reprints for
Governmental purposes notwithstanding any copyright notation thereon.
The views and conclusions contained herein are those of the authors and should
not be interpreted as necessarily representing the official policies or
endorsements, either expressed or implied, of DARPA or the U.S.\ Government.
Any opinions, findings, and conclusions or recommendations expressed in this
material are those of the authors and do not necessarily reflect the views of
DARPA, AFRL, NSF, NIH, ONR, or the U.S.\ government.

\clearpage
\bibliography{references}

\clearpage

\appendix

\section{Dense grid construction}
\label{sec:dense_grid_construction}

If we let $\int_{-\infty}^{\infty} \Lambda_i(\omega) f(\omega) \, d \omega \approx \sum_{l=1}^{L_i}
a_{i,l} f(\omega_{i,l})$ be the Gaussian quadrature rule for each integral, then we
can approximate $k$ with
\begin{align*}
  \tilde{k}(u)
  &= \prod_{i=1}^d \sum_{l=1}^{L_k} a_{i,l} \exp(j \omega_{i,l} e_i^\top u). 
\end{align*}
If we define $a_{\mathbf{l}} = \prod_{i=1}^d a_{i,l_i}$ and $\omega_{\mathbf{l}} =
\sum_{i=1}^d \omega_{i,l_i} e_i$ then we are left with the tensor product quadrature
rule
\begin{equation}
  \label{eqnTensorProductRule}
  \tilde{k}(u) = \sum_{ \mathbf{l} \in \prod_{i=1}^d \{1 \ldots L_i \} } a_{\mathbf{l}}
  \exp\left(j \omega_{\mathbf{l}}^\top u
  \right)
\end{equation}
over $D = \prod L_i$ points --- we can simplify this to $L^d$ in the case where every
$L_i = L$.

\section{Sparse grid construction}
\label{sec:sparse_grid_construction}

Here, we briefly describe the sparse grid construction.
We start by letting let $G_i^L(u_i)$ be the approximation of $k_i(u_i)$ that results
from applying the one-dimensional Gaussian quadrature rule with $L$ points: for
the appropriate sample points and weights,
\[
  G_i^L(u_i) = \sum_{l=1}^L a_l \exp(j u_i \omega_l).
\]
One of the properties of Gaussian quadrature is that it is exact in the limit of
large $L$.
In particular, this limit means that we can decompose $k_i(u_i)$ as the infinite
sum
\begin{equation*}
  k_i(u_i) = G_i^1(u_i) + \sum_{m=1}^{\infty} \left(G_i^{2^m}(u_i) - G_i^{2^{m-1}}(u_i) \right) =
  \sum_{m=0}^{\infty} \Delta_{i,m}(u_i),
\end{equation*}
where $\Delta_{i,m}(u_i) = G_i^{2^m}(u_i) - G_i^{2^{m-1}}(u_i)$.
To represent $k(u)$, it suffices to use the product
\[
  k(u) = \sum_{\mathbf{m} \in \N^d} \prod_{i=1}^d \Delta_{i,m_i}(u_i) = \sum_{\mathbf{m} \in \N^d}
  \Delta_{\mathbf{m}}(u)
\]
where $\Delta_{\mathbf{m}}(u) = \prod_{i=1}^d \Delta_{i, m_i}(u_i)$.  We can think of these
$\Delta_{\mathbf{m}}$ forming a ``grid'' of terms in $\N^d$.  We plot this
grid for $d = 2$ in Figure \ref{figGridDiagram}.  The dense grid approximation
is equivalent to summing up a hypercube of these terms, which we illustrate
as a square in the figure.

\begin{figure}[h]%
\centering
\resizebox{.85\columnwidth}{!}{%
\begin{tikzpicture}[every node/.style={inner sep=0,outer sep=0}]
  \foreach \x in {0,...,5}
    \foreach \y in {0,...,3}
      \draw (\x, \y) node[draw,circle,fill=black,minimum size=0.2cm] {};
  \foreach \x in {0,...,5}
    \draw (\x, 4) node {$\vdots$};
  \foreach \y in {0,...,3}
    \draw (6, \y) node {$\cdots$};
  \draw (6, 4) node {$\iddots$};
  \draw[red,very thick,dashed] (-0.3,-0.3) -- (3.3, -0.3)
    -- (3.3, 3.3) -- (-0.3, 3.3) -- cycle;
  \draw[blue,very thick] (-0.2,-0.2) -- (3.5, -0.2)
     -- (-0.2, 3.5) -- cycle;
  \draw[red,very thick,dashed] (6.5,2.0) -- (7.5,2.0)
    node[black,right] {\hspace{5pt}dense grid};
  \draw[blue,very thick] (6.5,1.5) -- (7.5,1.5)
    node[black,right] {\hspace{5pt}sparse grid};
\end{tikzpicture}}
\caption{
Grid of approximation terms $\Delta_{\mathbf{m}}$.
}
\label{figGridDiagram}
\end{figure}
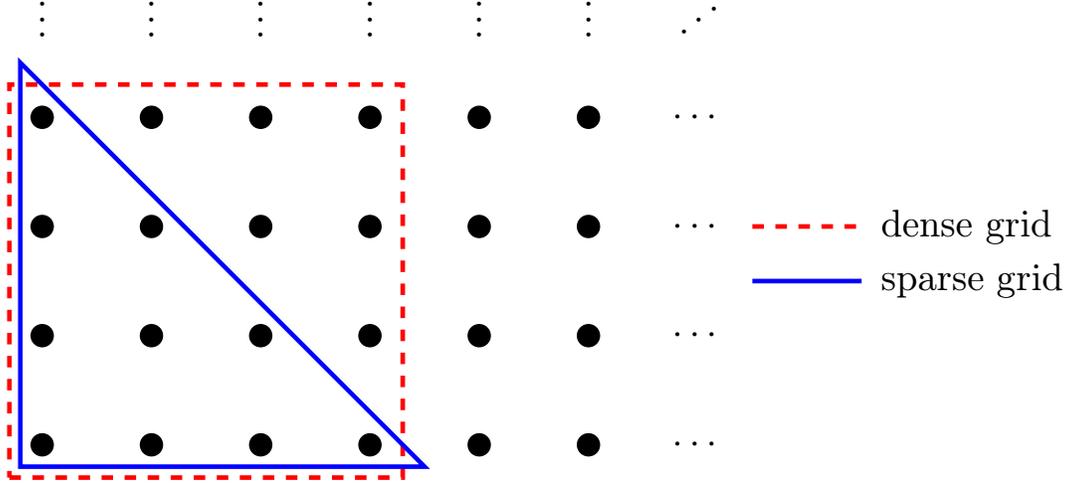

Smolyak's sparse grid approximation
approximates this sum by using only those $\Delta_{\mathbf{m}}$ that can
be computed with a ``small'' number of samples.  Specifically, the sparse
grid up to level $A$ is defined as,
\[
  \tilde{k}(u) = \sum_{\mathbf{m} \in \N^d, \, \mathbf{1}^\top \mathbf{m} \le A} \Delta_{\mathbf{m}}(u).
\]
In Figure \ref{figGridDiagram}, this is illustrated by the blue triangle ---
the efficiency of sparse grids comes from the fact that in higher dimensions,
the simplex of terms used by the sparse grid contains exponentially (in $d$)
fewer quadrature points than the hypercube of terms used by a dense grid.

Now, for any $u$, each $\Delta_{\mathbf{m}}(u)$ can be computed using the tensor
product quadrature rule from (\ref{eqnTensorProductRule}); the number of samples
required is no greater than $3^{\mathbf{1}^\top \mathbf{m}}$.
Combining this with the previous equation gives us a rough upper bound on the
sample count of the sparse grid construction
\[
  D \le \sum_{\mathbf{m} \in \N^d, \, \mathbf{1}^\top \mathbf{m} \le A} 3^{\mathbf{1}^\top
    \mathbf{m}} \le 3^A {d + A \choose A}.
\]

\section{Proofs}
\label{sec:proofs}

\subsection{Proof of Theorem~\ref{thmPolynomialQuadrature}}

In order to prove this theorem, we will need a couple of lemmas.

\begin{lemma}[Stirling's Approximation]
  \label{lemmaStirling}
  For any positive integer $n$,
  \[
    \left(n + \frac{1}{2} \right) \log n - n + \frac{1}{2} \log (2 \pi) + \frac{1}{12n+1}
    \le
    \log n!
    \le
    \left(n + \frac{1}{2} \right) \log n - n + \frac{1}{2} \log (2 \pi) + \frac{1}{12n}.
  \]
\end{lemma}

\begin{lemma}[Subgaussian Moment Bound]
  \label{lemmaSubgaussianMomentBound}
  If a random variable $X$ is $b$-subgaussian, then its $p$-th moment is bounded
  by
  \[
    \E{\norm{X}^p}
    \le
    p 2^{\frac{p}{2}} b^p \Gamma \left( \frac{p}{2} \right).
  \]
\end{lemma}

We now prove the theorem.

\begin{proof}[Proof of Theorem \ref{thmPolynomialQuadrature}]

For any $x$, define $\epsilon(x)$, the error function, as
\begin{equation*}
  \epsilon(x) = \abs{k(x) - \sum_{i=1}^D a_i \exp(j x^\top \omega_i)}.
\end{equation*}
By Taylor's theorem, there exists a function $\beta(z)$ such that
\begin{equation*}
  \exp(j z) = \sum_{k=0}^{R-1} \frac{(j z)^k}{k!} + \frac{(j z)^R}{R!} \exp(j
  \beta(z)).
\end{equation*}
This is the mean value theorem form for the Taylor series remainder.
Therefore we can write $\epsilon(x)$ as
{
\small
\begin{align*}
  &\hspace{-2em}\abs{k(x) - \sum_{i=1}^D a_i \exp(j x^\top \omega_i)} \\
  =& \abs{\int \Lambda(\omega) \exp(j x^\top \omega) d \omega - \sum_{i=1}^D a_i \exp(j x^\top \omega_i)} \\
  =& \abs{\int \Lambda(\omega) \left(\sum_{l=0}^{R-1} \frac{(j x^\top \omega)^l}{l!} + \frac{(j x^\top
    \omega)^R}{R!} e^{j \beta(x^\top \omega)} \right) d \omega - \sum_{i=1}^D a_i \left(\sum_{l=0}^{R-1}
    \frac{(j x^\top \omega_i)^l}{l!} + \frac{(j x^\top \omega_i)^R}{R!} e^{j \beta(x^\top \omega_i)}
    \right)} \\
  =& \abs{\sum_{l=0}^{R-1} \frac{j^l}{l!} \left( \int \Lambda(\omega) (x^\top \omega)^l d \omega - \sum_{i=1}^D a_i
     (x^\top \omega_i)^l \right) + \frac{j^R}{R!} \left( \int \Lambda(\omega) (x^\top \omega)^R e^{j \beta(x^\top \omega)}
    d \omega - \sum_{i=1}^D a_i (x^\top \omega_i)^R e^{j \beta(x^\top \omega_i)} \right)}.
\end{align*}
}
Now, since our quadrature is exact up to degree $R$, by the condition from
(\ref{eqnPolyExactConstraints}), the first term is zero:
\begin{align*}
  \epsilon(x)
  &= \abs{\frac{j^R}{R!} \left( \int \Lambda(\omega) (x^\top \omega)^R \exp(j \beta(x^\top \omega)) d \omega - \sum_{i=1}^D
      a_i (x^\top \omega_i)^R \exp(j \beta(x^\top \omega_i)) \right) } \\
  &\le \frac{1}{R!} \left(\int \abs{\Lambda(\omega) (x^\top \omega)^R \exp(j \beta(x^\top \omega))} d \omega + \sum_{i=1}^D
    \abs{a_i (x^\top \omega_i)^R \exp(j \beta(x^\top \omega_i))} \right) \\
  &\le \frac{1}{R!} \left(\int \abs{\Lambda(\omega) (x^\top \omega)^R} d \omega + \sum_{i=1}^D \abs{a_i (x^\top
      \omega_i)^R} \right).
\end{align*}
Since $R$ is even, $a_i \ge 0$, and $\Lambda(\omega) \ge 0$,
\begin{equation*}
  \epsilon(x) \le \frac{1}{R!} \left(\int \Lambda(\omega) (x^\top \omega)^R d \omega + \sum_{i=1}^D a_i (x^\top \omega_i)^R
  \right).
\end{equation*}
Again applying our condition from (\ref{eqnPolyExactConstraints}),
\begin{equation*}
  \epsilon(x) \le \frac{2}{R!} \int \Lambda(\omega) (x^\top \omega)^R d \omega.
\end{equation*}
Finally, by Cauchy-Schwarz,
\begin{equation*}
  \epsilon(x) \le \frac{2 \norm{x}^R}{R!} \int \Lambda(\omega) \norm{\omega}^R d \omega \le \frac{2 \norm{x}^R}{R!}
  \mathbf{E}_{\Lambda}\left[ \norm{\omega}^R \right].
\end{equation*}
Now, since we assumed that $\Lambda$ was $b$-subgaussian, we can apply Lemma
\ref{lemmaSubgaussianMomentBound} to bound this expected value with
\begin{equation*}
  \epsilon(x) \le \frac{2 \norm{x}^R}{R!} \int \Lambda(\omega) \norm{\omega}^R d \omega \le \frac{2 \norm{x}^R}{R!}
  R 2^{\frac{R}{2}} b^R \Gamma \left( \frac{R}{2} \right) = 4 b^R \norm{x}^R
  \frac{2^{R/2} (R/2)!}{R!}.
\end{equation*}
Now we need to bound $\frac{2^{R/2} (R/2)!}{R!}$ using Stirling's
approximation (Lemma~\ref{lemmaStirling}):
\begin{align*}
  &-\log (R!) + \frac{R}{2} \log 2 + \log \left( (R/2)! \right) \\
  \le& - \left(\left(R + \frac{1}{2} \right) \log R - R + \frac{1}{2} \log (2 \pi) +
     \frac{1}{12R + 1}
     \right) + \frac{R}{2} \log 2 + \left( \frac{R}{2} + \frac{1}{2} \right)
     \log \left(\frac{R}{2} \right) - \frac{R}{2} + \frac{1}{2} \log (2 \pi) + \frac{1}{6R} \\
  =& - R \log R - \frac{1}{2} \log R + \frac{R}{2}  - \frac{1}{12R + 1} +
     \frac{R}{2} \log 2 + \frac{R}{2} \log R + \frac{1}{2} \log R - \frac{R}{2}
     \log 2 - \frac{1}{2} \log 2 + \frac{1}{6R} \\
  =& \ \frac{R}{2} - \frac{1}{2} \log 2 - \frac{R}{2} \log R + \frac{1}{6R} -
     \frac{1}{12R + 1} \\
  \leq& \ -\frac{1}{2} \log 2 + \frac{1}{12} - \frac{1}{25} + \frac{R}{2}
     \left( 1 - \log R \right),
\end{align*}
where we have used the fact that $R \geq 2$ since $R$ is even.
Taking the exponential results in
\begin{equation*}
  \epsilon(x) \le 4 b^R \norm{x}^R \frac{e^{1/12-1/25}}{\sqrt{2}} \left( \frac{e}{R} \right)^{R/2} \leq
  3 \left(\frac{e b^2 \norm{x}^2}{R} \right)^{\frac{R}{2}}.
\end{equation*}
Therefore, for any $x, y \in \mathcal{M}$,
\begin{align*}
  \abs{k(x - y) - \tilde{k}(x - y)}
  &= \abs{k(x - y) - \sum_{i=1}^D a_i \exp(j (x - y)^\top \omega_i)} \\
  &= \epsilon(x - y) \\
  &\le 3 \left(\frac{e b^2 \norm{x - y}^2}{R}
    \right)^{\frac{R}{2}}.
\end{align*}
Finally, since $\mathcal{M}$ has diameter $M$, we know that
$\norm{x - y} \le M$, so we can conclude that
\begin{equation*}
  \abs{k(x, y) - \tilde{k}(x - y)} \le 3 \left(\frac{e b^2 M^2}{R}
  \right)^{\frac{R}{2}},
\end{equation*}
which is the desired expression.

\end{proof}

Using this, we can directly prove Corollary~\ref{corPolynomialQuadrature}.
\begin{proof}[Proof of Corollary~\ref{corPolynomialQuadrature}]
  By assumption, the number of samples required is
  \[
    D \le \beta { d + R \choose d } \le \beta 2^d \cdot 2^R \le \beta 2^d \exp(R).
  \]

  In order to ensure
  \[
    \sup_{\norm{u} \leq M} \abs{k(u) - \tilde{k}(u)}
    \le
    \epsilon,
  \]
  it suffices by the result of Theorem~\ref{thmPolynomialQuadrature} to have $R$
  large enough that
  \[
    3 \left( \frac{e b^2 M^2}{R} \right)^{R/2} \le \epsilon \qquad \text{and} \qquad \frac{e b^2 M^2}{R} < 1.
  \]
  Suppose that we set $R$ such that
  \[
    \frac{e b^2 M^2}{R} \le \exp(-2 \gamma).
  \]
  If $\gamma > 0$, then the second condition will be trivially satisfied.
  The first condition will also be satisfied when we set $R$ large enough that
  \[
    3\exp(-\gamma R) \le \epsilon.
  \]
  This occurs when
  \[
    \exp(R) \ge \left( \frac{3}{\epsilon} \right)^{1/\gamma}.
  \]
  For this condition to hold, as $D \leq \beta 2^d \exp(R)$, it suffices to have
  \[
    D \geq \beta 2^d \left( \frac{3}{\epsilon} \right)^{1/\gamma}
  \]
  samples.  On the other hand, for $R$ to satisfy our original condition,
  we also need
  \[
    R \ge e b^2 M^2 \exp(2 \gamma).
  \]
  This can be achieved when
  \[
    D \geq \beta 2^d \exp(e^{2\gamma+1} b^2 M^2).
  \]
  Combining these two conditions using a maximum proves the corollary.
\end{proof}

We can similarly prove Corollary~\ref{corDenseGridQuadrature}.
\begin{proof}[Proof of Corollary~\ref{corDenseGridQuadrature}]

  For the quadrature rule to be exact for polynomials of degrees up to (even)
  $R$, it suffices for each one-dimensional rule to have $L = R/2 + 1$ points.
  The total number of points used is then $D = L^d = \exp (d \log (R/2 + 1))$.
  Since $R \geq 2$ as it is even, $\log (R/2 + 1) \leq R/2$, so $D \leq \exp(Rd/2)$.

  As in the proof of Corollary~\ref{corPolynomialQuadrature}, to ensure
  $\sup_{\norm{u} \leq M} \abs{k(u) - \tilde{k}(u)} \le \epsilon$, it suffices by the result
  of Theorem~\ref{thmPolynomialQuadrature} to have $R$ large enough that
  \[
    3 \left( \frac{e b^2 M^2}{R} \right)^{R/2} \le \epsilon
    \qquad \text{and } \qquad
    \frac{e b^2 M^2}{R} < 1.
  \]
  Suppose that we set $R$ such that
  \[
    \frac{e b^2 M^2}{R} \le \exp(-d\gamma).
  \]
  If $\gamma > 0$, then the second condition will be trivially satisfied.
  The first condition will also be satisfied when we set $R$ large enough that
  \[
    3\exp(-\gamma Rd/2) \le \epsilon.
  \]
  This occurs when
  \[
    \exp(Rd/2) \ge \left( \frac{3}{\epsilon} \right)^{1/\gamma}.
  \]
  For this condition to hold, as $D \leq \exp(Rd/2)$, it suffices to have
  \[
    D \geq \left( \frac{3}{\epsilon} \right)^{1/\gamma}
  \]
  samples.  On the other hand, for $R$ to satisfy our original condition,
  we also need
  \[
    R \ge e b^2 M^2 \exp(d\gamma).
  \]
  This can be achieved when
  \[
    D
    \geq
    \exp(de^{d\gamma} e b^2 M^2/2).
  \]
  Combining these two conditions using a maximum proves the corollary.
\end{proof}
\subsection{Proof of Theorem~\ref{thmSparseGridQuadrature}}

\begin{proof}[Proof of Theorem \ref{thmSparseGridQuadrature}]
  Based on the construction of the sparse grid in
  Section~\ref{sec:sparse_grid_construction}, $k$ and $\tilde{k}$ differ in the
  terms $\sum_{\mathbf{m} \in \mathbb{N}^d, \1^\top \mathbf{m} > A} \Delta_{\mathbf{m}} (u)$.
  Thus we need to bound the error
  \begin{equation*}
    \sup_{\norm{u} \leq M} \abs{k(u) - \tilde{k}(u)} = \sup_{\norm{u} \leq M}
    \abs{\sum_{\mathbf{m} \in \mathbb{N}^d, \1^\top \mathbf{m} > A} \Delta_{\mathbf{m}} (u)}
    \leq \sum_{\mathbf{m} \in \mathbb{N}^d, \1^\top \mathbf{m} > A} \sup_{\norm{u} \leq M} \abs{\Delta_{\mathbf{m}} (u)}.
  \end{equation*}
  But $\Delta_{\mathbf{m}}(u)$ is just a product of one-dimensional rules, and we can
  apply Theorem~\ref{thmPolynomialQuadrature} for each dimension.
  Indeed, the Gaussian quadrature rule with $L$ points $G_i^L$ is exact for
  polynomials of degree up to $2L - 1$, so the bound from
  Theorem~\ref{thmPolynomialQuadrature} with $R = 2(L - 1)$ becomes $3 \left(
    \frac{e b^2 u_i^2}{2(L - 1)} \right)^{L - 1} \leq 3 \left( \frac{e b^2
      u_i^2}{L} \right)^{L - 1}$ (since $2(L - 1) \geq L$).
  As $\Delta_{i, m_i} (u_i) = G_i^{2^m_i}(u_i) - G_i^{2^{m_i-1}}(u_i)$, we have
  \begin{align*}
    \abs{\Delta_{i, m_i}(u_i)}
    &= \abs{G_i^{2^m_i}(u_i) - G_i^{2^{m_i - 1}}(u_i)} \\
    &\leq \abs{G_i^{2^m_i}(u_i) - k_i(u_i)} + \abs{k_i(u_i) - G_i^{2^{m_i - 1}}(u_i)} \\
    &\leq 3 \left(eb^2u_i^2\right)^{2^{m_i} - 1} 2^{-m_i (2^{m_i} - 1)} + 3
      \left(eb^2u_i^2\right)^{2^{m_i - 1} - 1} 2^{-(m_i-1) (2^{m_i-1} - 1)} \\
    &\leq 6 \left(eb^2u_i^2\right)^{2^{m_i} - 1} 2^{-(m_i-1) (2^{m_i-1} - 1)} \\
    &= 6 (\sqrt{e} b u_i)^{2^{m_i}} 2^{-(m_i - 1)2^{m_i -1}} 2^{m_i-1}.
  \end{align*}
  If we let $c_i = 2^{m_i - 1}$ (and $c_i = 0$ if $m_i = 0$), then we can
  rewrite this as
  \begin{equation*}
    \abs{\Delta_{i, m_i}(u_i)}
    \le
    \frac{6 (\sqrt{e} b)^{2 c_i} c_i}{c_i^{c_i}} u_i^{2 c_i}.
  \end{equation*}
  Thus
  \begin{equation*}
    \abs{\Delta_{\mathbf{m}}(u)}
    \le
    \prod_{i \in \{ 1 \ldots d \}, m_i > 0}
    \frac{6 (\sqrt{e} b)^{2 c_i} c_i}{c_i^{c_i}} u_i^{2 c_i}.
  \end{equation*}
  As $6c_i \leq 6^{c_i}$, we have $\abs{\Delta_{\mathbf{m}}(u)} \le \prod_{i \in \{ 1 \ldots d \},
    m_i > 0} \frac{(\sqrt{6e} b)^{2 c_i}}{c_i^{c_i}} u_i^{2 c_i}$.
  Next, applying Lemma \ref{lemmaZBound} gives
  \begin{align*}
    \abs{\Delta_{\mathbf{m}}(u)}
    &\le
      \left(
      \prod_{i \in \{ 1 \ldots d \}, m_i > 0}
      \frac{(\sqrt{6e} b)^{2 c_i} }{c_i^{c_i}}
      \right)
      \left(
      M^{2\norm{c}_1}
      \norm{c}_1^{-\norm{c}_1}
      \prod_{i \in \{ 1 \ldots d \}, m_i > 0} c_i^{c_i}
      \right) \\
    &=
      (\sqrt{6e} b)^{2 \norm{c}_1}
      M^{2 \norm{c}_1}
      \norm{c}_1^{-\norm{c}_1} \\
    &=
      (6e b^2 M^2)^{\norm{c}_1}
      \norm{c}_1^{-\norm{c}_1}.
  \end{align*}
  Since $\norm{c}_1 \ge \norm{m}_1 \ge A$, we can bound the error term with
  \begin{align*}
    \sup_{\norm{u} \leq M} \abs{k(u) - \tilde{k}(u)}
    &\leq \sum_{\mathbf{m} \in
            \mathbb{N}^d, \1^\top \mathbf{m} > A} \sup_{\norm{u} \leq M} \abs{\Delta_{\mathbf{m}}
            (u)} \\
    &\leq
    \sum_{\mathbf{m}: \norm{\mathbf{m}}_1 > A}
    (6e b^2 M^2)^{\norm{c}_1}
    \norm{c}_1^{-\norm{c}_1} \\
    &\le
    \sum_{\mathbf{m}: \norm{\mathbf{m}}_1 > A}
    \left( \frac{6e b^2 M^2}{A} \right)^{\norm{c}_1}.
  \end{align*}
  Now, since by assumption $A \ge 24e b^2 M^2$ and $\norm{c}_1 \ge \norm{m}_1$, it
  follows that $6e b^2 M^2 / A \le 1$ and so we can upper-bound this sum with
  \begin{align*}
    \sup_{\norm{u} \leq M} \abs{k(u) - \tilde{k}(u)}
    &\le
    \sum_{\mathbf{m}: \norm{\mathbf{m}}_1 > A}
    \left( \frac{6e b^2 M^2}{A} \right)^{\norm{\mathbf{m}}_1} \\
    &=
    \sum_{l = A + 1}^{\infty}
    \sum_{\mathbf{m}: \norm{\mathbf{m}}_1 = l}
    \left( \frac{6e b^2 M^2}{A} \right)^l \\
    &=
    \sum_{l = A + 1}^{\infty}
    {d + l - 1 \choose l}
    \left( \frac{6e b^2 M^2}{A} \right)^l \\
    &\le
    \sum_{l = A + 1}^{\infty}
    2^{d + l - 1}
    \left( \frac{6e b^2 M^2}{A} \right)^l \\
    &=
    2^{d-1}
    \sum_{l = A + 1}^{\infty}
    \left( \frac{12e b^2 M^2}{A} \right)^l.
  \end{align*}
  Summing this geometric series results in
  \begin{align*}
    \sup_{\norm{u} \leq M} \abs{k(u) - \tilde{k}(u)}
    &\le
    2^{d-1}
    \left( \frac{12e b^2 M^2}{A} \right)^{A + 1}
    \left( 1 - \frac{12e b^2 M^2}{A} \right)^{-1} \\
    &\le
    2^{d-1}
    \left( \frac{12e b^2 M^2}{A} \right)^A
    \left( 1 - \frac{1}{2} \right)^{-1} \\
    &=
    2^d
    \left( \frac{12e b^2 M^2}{A} \right)^A.
  \end{align*}
  This is what we wanted to show.
\end{proof}

Using this, we can directly prove Corollary~\ref{corSparseGridQuadrature}.
\begin{proof}[Proof of Corollary~\ref{corSparseGridQuadrature}]
  Recall that the number of samples required for a sparse grid rule up to order $A$ is
  \[
    D \le 3^A { d + A \choose A } \le 2^d \cdot 6^A \le 2^d \exp(2 A).
  \]

  In order to ensure
  \[
    \sup_{\norm{u} \leq M} \abs{k(u) - \tilde{k}(u)}
    \le
    \epsilon,
  \]
  it suffices by the result of Theorem~\ref{thmSparseGridQuadrature} to have $A$
  large enough that
  \[
    2^d
    \left( \frac{12e b^2 M^2}{A} \right)^A
    \le
    \epsilon
  \]
  and
  \[
    \frac{12e b^2 M^2}{A} < 1.
  \]
  Suppose that we set $A$ such that
  \[
    \frac{12e b^2 M^2}{A} \le \exp(-2 \gamma).
  \]
  If $\gamma > 0$, then the second condition will be trivially satisfied.
  The first condition will also be satisfied when we set $A$ large enough that
  \[
    2^d
    \exp(-2 \gamma A)
    \le
    \epsilon.
  \]
  This occurs when
  \[
    \exp(2 A)
    \ge
    2^{d/\gamma}
    \cdot
    \epsilon^{-1/\gamma}.
  \]
  For this condition to hold, as $D \leq 2^d \exp(2A)$, it suffices to have
  \[
    D
    \geq
    2^d
    \cdot
    2^{d/\gamma}
    \cdot
    \epsilon^{-1/\gamma}
  \]
  samples.  On the other hand, for $A$ to satisfy our original condition,
  we also need
  \[
    A \ge 12e b^2 M^2 \exp(2 \gamma).
  \]
  This can be achieved when
  \[
    D
    \geq
    2^d \exp(24e b^2 M^2 \exp(2 \gamma)).
  \]
  Combining these two conditions using a maximum proves the corollary.
\end{proof}

We now prove the technical lemma that we have used.
\begin{lemma}
  \label{lemmaZBound}
  For any $u \in \R^d$ that satisfies $\norm{u} \le M$, and any $c \in \R^d$
  with $c_i > 0$,
  \[
    \prod_{i=1}^d u_i^{2 c_i}
    \le
    M^{2 \norm{c}_1}
    \norm{c}_1^{-\norm{c}_1}
    \prod_{i=1}^d c_i^{c_i}.
  \]
\end{lemma}
\begin{proof}
  We produce this result by optimizing over $u_i$.  First, we let $x_i = u_i^2$,
  and note that an upper bound is
  \[
    \prod u_i^{2 c_i}
    \le
    \max_{\sum x_i = M^2} \prod_{i=1}^d x_i^{c_i}.
  \]
  Taking the logarithm and using the method of Lagrange multipliers to handle
  the constraint, we get Lagrangian
  \[
    J(x, u)
    =
    \sum_{i=1}^d c_i \log(x_i) + u \left( M^2 - \sum_{i=1}^d x_i \right).
  \]
  Differentiating to minimize gets us, for all $i$,
  \[
    0
    =
    \frac{c_i}{x_i} - u.
  \]
  which results in
  \[
    x_i = \frac{c_i}{u}.
  \]
  In order to satisfy the constraint, we must set $u$ such that
  \[
    x_i
    =
    \frac{c_i M^2}{\sum_{j=1}^d c_i}
    =
    \frac{c_i M^2}{\norm{c}_1}.
  \]
  With this assignment, we have
  \[
    \prod u_i^{2 c_i}
    \le
    \prod_{i=1}^d \left( \frac{c_i M^2}{\norm{c}_1} \right)^{c_i},
  \]
  and simplification produces the desired result.
\end{proof}

\subsection{Proof of Theorem~\ref{thmSparseANOVAKernel}}

\begin{proof}
  We first bound the approximation error of each sub-kernel acting on a subset
  $S$ of indices.
  Let $x_S, y_S$ be the vector $x$, $y$ restricted to these indices and let
  $k_S(x_S - y_S)$ be the sub-kernel acting on indices in $S$.
  As shown at the end of the proof of Theorem \ref{thmPolynomialQuadrature},
  \begin{equation*}
    \abs{k_S(x_S - y_S) - \tilde{k}_S(x_S - y_S)}
    \le 3 \left(\frac{e b^2 \norm{x_S - y_S}^2}{R} \right)^{\frac{R}{2}}
    = 3 \left(\frac{e b^2}{R} \right)^{\frac{R}{2}} \norm{x_S - y_S}^R.
  \end{equation*}
  As $\mathcal{M}$ has diameter $M$, $\norm{x_S - y_S} \leq M$.
  Noting that $R \geq 2$ since it is even, we can bound $\norm{x_S - y_S}^R \leq
  \frac{\norm{x_S - y_S}^2}{M^2} M^R$, and so
  \begin{equation*}
    \abs{k_S(x_S - y_S) - \tilde{k}_S(x_S - y_S)}
    \le 3 \frac{\norm{x_S - y_S}^2}{M^2} \left(\frac{e b^2 M^2}{R} \right)^{\frac{R}{2}}.
  \end{equation*}
  Summing over all $m$ sub-kernels, noting that each index only appears in at
  most $\Delta$ sets $S$, we have that $\sum_{S \in \mathcal{S}} \norm{x_S - y_S}^2 \leq
  \Delta \norm{x - y}^2 \leq \Delta M^2$.
  Therefore
  \begin{equation*}
    \epsilon = \abs{k(x - y) - \tilde{k}(x - y)} \leq 3 \Delta \left( \frac{e b^2 M^2}{R} \right)^{R/2}.
  \end{equation*}

  The number of points we will use in total is $D \leq m \beta {r + R \choose r}$.
  By a similar argument as in Corollary~\ref{corPolynomialQuadrature}, for any
  $\gamma > 0$, we obtain
  \begin{equation*}
    D(\epsilon) \leq \beta m 2^r \max \left( \exp(e^{2\gamma+1} b^2 M^2), (3\Delta)^{1/\gamma} \epsilon^{-1/\gamma} \right).
  \end{equation*}
\end{proof}

\section{Details of experiments}

For the MNIST dataset, we use a 60k/10k split of train and test set.
We use linear SVM on top of the features generated by random Fourier features or
subsampled dense grid.
The kernel bandwidth and the SVM hyper-parameters are chosen by cross
validation.

For the task of acoustic modeling on the TIMIT dataset, the input features
correspond to a frame of 10ms of speech, preprocessed using the standard feature
pace Maximum Likelihood Linear Regression (fMMLR).
The input dimension is 40.
We use multinomial logistic regression on top of the features generated by
random Fourier features, QMC, subsampled dense grid, or reweighted dense grid.
The output of the multinomial logistic regressions is a probability distribution
over 1917 groups of tri-phonemes.
We constrain the weight matrix of the logistic regression to have rank at most
500, similar to \cite{may2017kernel}.
The kernel bandwidth is chosen by performance on the validation set.

\end{document}